\documentclass[letterpaper, 10 pt, conference]{ieeeconf}  

\IEEEoverridecommandlockouts                              

\title{\LARGE \bf
Feedback-based Digital Higher-order Terminal Sliding Mode \\ for 6-DOF Industrial Manipulators
}

\author{Zhian Kuang$^{1,2}$, Xiang Zhang$^{2}$, Liting Sun$^{2}$, Huijun Gao$^{1}$ and Masayoshi Tomizuka$^{2}$
\thanks{$^{1}$ Zhian Kuang and Huijun Gao are with the Research Institute of Intelligent Control and Systems, Harbin Institute of Technology, 150001, Harbin, P.R. China. {\tt\small zhiankuang@berkeley.edu, hjgao@hit.edu.cn}}%
\thanks{$^{2}$ Xiang Zhang, Liting Sun, and Masayoshi Tomizuka are with the Department of Mechanical Engineering, University of California, Berkeley, CA 94720, USA.
{\tt\small xiang\_zhang\_98@berkeley.edu, litingsun@berkeley.edu, tomizuka@berkeley.edu}}%

\thanks{$^{2}$
Zhian Kuang is also with the Department of Mechanical Engineering, University of California, Berkeley, CA 94720, USA.}
}
\usepackage[T1]{fontenc}
\usepackage{aecompl}
\usepackage{cite}
\usepackage{arydshln}

\usepackage{amsmath}
\usepackage{xcolor}
\usepackage{bm}
\usepackage{color}

\usepackage{graphicx}
\usepackage{float}
\usepackage{amssymb}
\usepackage{cases}
\usepackage{hyperref}
\usepackage{hhline}
\usepackage{balance}

\usepackage{multirow}

\newtheorem{myrem}{Remark}
\newtheorem{mythm}{Theorem}

\begin{document}
\maketitle
\thispagestyle{empty}
\pagestyle{empty}

\begin{abstract}

The precise motion control of a multi-degree of freedom~(DOF) robot manipulator is always challenging due to its nonlinear dynamics, disturbances, and uncertainties. Because most manipulators are controlled by digital signals, a novel higher-order sliding mode controller in the discrete-time form with time delay estimation is proposed in this paper. The dynamic model of the manipulator used in the design allows
proper handling of nonlinearities, uncertainties and disturbances involved in the problem. Specifically, parametric uncertainties and disturbances are handled by the time delay estimation and the nonlinearity of the manipulator is addressed by the feedback structure of the controller. The combination of terminal sliding mode surface and higher-order control scheme in the controller guarantees a fast response with a small chattering amplitude. Moreover, the controller is designed with a modified sliding mode surface and variable-gain structure, so that the performance of the controller is further enhanced. We also analyse the condition to guarantee the stability of the closed-loop system in this paper. Finally, the simulation and experimental results prove that the proposed control scheme has a precise performance in a robot manipulator system.

\end{abstract}

\section{INTRODUCTION}

Manipulators have been extensively used in the industry to manufacture products. However, with nonlinearities, high couplings, and significant parametric uncertainties existing in the dynamics, the manipulator's precise control is always challenging~\cite{wang2020model}. Moreover, with more and more industrial scenarios involving manipulators, different kinds of disturbances and high accelerations are applied, making it harder to design a precise and robust controller for a manipulator. 

For the past several years, researchers have proposed several kinds of manipulator controllers, such as feedforward and PID control~\cite{abdelhedi2014nonlinear}, sliding mode control~\cite{yang1999sliding}, adaptive robust control~\cite{alakshendra2016robust}. Among these methods, sliding mode control is proven to be straightforward to be applied in complex electromechanical system~\cite{park2000sliding}. More importantly, because of the sliding mode surface, this control method shows outstanding robustness to disturbances and uncertainties~\cite{kuang2020precise}. Researchers have proposed several kinds of modified SMC to make up the original drawbacks like the slow convergence on the sliding surface~\cite{ma2020fuzzy, kuang2020fractional}. For instance, Yong proposed a non-singular terminal sliding mode control (TSMC) method, which makes the states converge in finite time on the sliding surface~\cite{feng2002non}. 

Despite that TSMC has precise and robust performance in theory, it cannot be applied directly to the manipulators for two reasons. One reason is that the digital processor is always used to control the manipulator nowadays, directly applying the continuous-time controller to the discrete-time environment results in large chattering amplitude or even instability~\cite{kuang2018simplified}. To address this problem, researchers have put forward the discrete-time terminal sliding mode control~(DTSMC)~\cite{li2013discrete} and apply it to motion control systems~\cite{abidi2008discrete, kuang2018contouring}. Compared with the traditional linear sliding mode surface, the terminal sliding surface has a faster response benefiting from the finite-time convergence~\cite{tripathi2019fast, kim2019discrete}. Nevertheless, this kind of dynamics also enlarges the chattering amplitude, which makes the precision of the manipulator decrease. A proper way to solve this problem is to raise the order of the sliding mode~\cite{rath2016output,ahmed2019adaptive}. For example, Xu proposed a second-order terminal sliding mode strategy~\cite{xu2015piezoelectric}. However, the implementation is on a linear positioning model, which cannot be applied instantly to the manipulator. Moreover, the order of the sliding mode controller is fixed, which cannot be changed based on the performance of the practical application. Raising the order is beneficial to reduce the chattering, but may also reduce the controller's robustness. Thus we need to build a new and flexible structure that can balance the chattering amplitude and the robustness at the same time.

The other obstacles to improving manipulators' performance are the nonlinearity, the significant parametric uncertainties, and disturbances. They need to be handled appropriately; otherwise, the tracking precision will deteriorate, especially when the manipulator is in operation with large acceleration~\cite{iwamura2013motion}. A popular way is to introduce the feedforward term to compensate for the nonlinear dynamics and use the controller to suppress the uncertainties and disturbances~\cite{abdelhedi2014nonlinear}. As the feedforward structure is open-loop, this method is not sufficient enough when there are parametric uncertainties on the manipulator's dynamic model. 
Another method to handle the nonlinearity is the time-delay estimation (TDE)~\cite{lee2017adaptive,della2019unified}. By using TDE, the nonlinear dynamics, the present lumped disturbances, and uncertainties are estimated by the previous states~\cite{wang2020model}. 
When the sampling frequency of the system is high, the delay time is appropriately chosen, and the disturbances change relatively slowly, this method has excellent performance in~\cite{van2016finite} and \cite{kali2018super}. However, as TDE highly rely on the state at a particular time point, extra disturbances are easy to be introduced. Moreover, when the manipulator operates with fast speed and massive acceleration, the estimation error will increase due to the fast change of dynamics. Thus we need a new structure to make the best use of TDE to achieve better precision.

This paper proposes a novel feedback-based digital higher-order terminal sliding mode control (DHTSMC) scheme. This controller is designed based on the manipulator's nominal dynamics, including the linear and nonlinear parts. The manipulator's nonlinear dynamics are directly handled by the equivalent control part of the controller in a feedback way. As an assistant, TDE is used to compensate for the uncertainties and disturbances in the system. For the controller, a new terminal sliding mode surface is designed to improve the states' dynamics on the sliding mode surface. We also give out a universal and flexible form of a higher-order switching controller. A specific order can be selected to balance the chattering phenomenon and the robustness of the controller. To address the performance reduction when there are large accelerations, we mainly use the controller's variable gain. More importantly, we implement the proposed controller to a 6-DOF manipulator in the environments of both simulations and experiments, whose results show that our proposed method has an advantage over previous ones. 

The remainder of this paper is as follows. Section~\ref{sec:controller design} introduces the model of manipulators and the proposed controller. In Section~\ref{sec:stability analysis}, the stability of the closed-loop control system is analyzed. Section~\ref{sec:implementation and simulation} introduces the practical application of the proposed controller to manipulators and the results of simulations and experiments.

\section{CONTROLLER DESIGN} \label{sec:controller design}

The continuous-time Lagrange model of a manipulator with n-DOFs is written as
\begin{align} \label{model:natural model}
    \bm{M}(\bm{q})\bm{\ddot{q}}+\bm{C}(\bm{q},\bm{\dot{q}}) \bm{\dot{q}}+\bm{G}(\bm{\dot{q}})+\bm{F_f}(\bm{q},\bm{\dot{q}})+\bm{d}=\bm{\tau}
\end{align}
where $\bm{M}(\bm{q}) \in \mathbb{R}^{n\times n}$ is the inertia matrix, $\bm{C}(\bm{q},\bm{\dot{q}})\in \mathbb{R}^{n\times n}$ denotes the Coriolis matrix, $\bm{G(\dot{q})} \in \mathbb{R}^{n}$ and $\bm{F_f}(\bm{q},\bm{\dot{q}})\in \mathbb{R}^{n}$ are the gravity vector and the friction vector respectively, $\bm{\ddot{q}}$, $\bm{\dot{q}}$ and $\bm{q}$ are respectively the angular acceleration, velocity and position of each joint, $\bm{d}\in \mathbb{R}^{n}$ stands for the disturbances, and $\bm{\tau} \in \mathbb{R}^{n}$ is the torque vector.

Because the parametric uncertainties extensively exist in practical systems, we use $\bm{\bar{\bullet}}$ to denote the matrices or vectors calculated from the nominal values, then~(\ref{model:natural model}) is rewritten as
\begin{align}
    \bm{\tau}=&\bm{\bar{M}}(\bm{q})\bm{\ddot{q}}+\bm{\bar{C}}(\bm{q},\bm{\dot{q}}) \bm{\dot{q}}+\bm{\bar{G}}(\bm{\dot{q}})+\bm{\bar{F_f}}(\bm{q},\bm{\dot{q}})\nonumber \\
    &+\bm{H}(\bm{q},\bm{\dot{q}},\bm{\ddot{q}})
\end{align}
with
\begin{align}
    \bm{H}(\bm{q},\bm{\dot{q}},\bm{\ddot{q}})
    =&\bm{\tilde{M}}(\bm{q})\bm{\ddot{q}}+\bm{\tilde{C}}(\bm{q},\bm{\dot{q}}) \bm{\dot{q}}+\bm{\tilde{G}}(\bm{\dot{q}})
    \nonumber \\
    &+\bm{\tilde{F_f}}(\bm{q},\bm{\dot{q}})+\bm{d}
\end{align}
where $\tilde{\bullet}=\bullet-\bar{\bullet}$ stands for the uncertain part of each matrix or vector.

If the manipulator is sampled with a sampling interval $T$, then the discrete-time model at time $t$ is described as
\begin{align} \label{model:discrete-time}
    \bm{\tau_k}=&\bm{\bar{M}}(\bm{q_k})\bm{\ddot{q_k}}+\bm{\bar{C}}(\bm{q_k},\bm{\dot{q}_k}) \bm{\dot{q}_k}+\bm{\bar{G}}(\bm{\dot{q}_k})+\bm{\bar{F_f}}(\bm{q}_k,\bm{\dot{q}}_k)
    \nonumber \\ &+\bm{H}(\bm{q_k},\bm{\dot{q}_k},\bm{\ddot{q}_k})
\end{align}
with
\begin{align}
    \bm{H}(\bm{q_k},\bm{\dot{q}_k},\bm{\ddot{q}_k})
    =&\bm{\tilde{M}}(\bm{q_k})\bm{\ddot{q}_k}+\bm{\tilde{C}}(\bm{q_k},\bm{\dot{q}_k}) \bm{\dot{q}_k}+\bm{\tilde{G}}(\bm{\dot{q}_k})
    \nonumber \\
    &+\bm{\tilde{F_f}}(\bm{q_k},\bm{\dot{q}_k})+\bm{d_k}
\end{align}
where $t=kT$, $\bm{\dot{q}_k}=\frac{\bm{q_{k+1}}-\bm{q_{k}}}{T}$, and $\bm{\ddot{q}_k}=\frac{\bm{\dot{q}_{k+1}}-\bm{\dot{q}_{k}}}{T}$.


We use TDE to estimate $\bm{H}(\bm{q_k},\bm{\dot{q}_k},\bm{\ddot{q}_k})$, i.e., we calculate out the uncertain part in the previous sampling point and use this value to approximate the uncertain part in the present. It is formulated as
\begin{align}
    \bm{\hat{H}}(\bm{q_k},\bm{\dot{q}_k},\bm{\ddot{q}_k})=&\bm{H}(\bm{q_{k-1}},\bm{\dot{q}_{k-1}},\bm{\ddot{q}_{k-1}}) \nonumber \\
    =&\bm{\tau_{k-1}}-\bm{\bar{M}}(\bm{q_{k-1}})\bm{\ddot{q}_{k-1}} \nonumber\\
    &-\bm{\bar{C}}(\bm{q_{k-1}},\bm{\dot{q}_{k-1}}) \bm{\dot{q}_{k-1}}
    \nonumber \\
    &-\bm{\bar{G}}(\bm{\dot{q}_{k-1}})-\bm{\bar{F_f}}(\bm{q}_{k-1},\bm{\dot{q}}_{k-1}).
\end{align}

The estimation error between $\bm{H}(\bm{q_k},\bm{\dot{q}_k},\bm{\ddot{q}_k})$ and $\bm{\hat{H}}(\bm{q_k},\bm{\dot{q}_k},\bm{\ddot{q}_k})$ is presented as
$\bm{\check{H}}(\bm{q_k},\bm{\dot{q}_k},\bm{\ddot{q}_k})=\bm{H}(\bm{q_k},\bm{\dot{q}_k},\bm{\ddot{q}_k})-\bm{H}(\bm{q_{k-1}},\bm{\dot{q}_{k-1}},\bm{\ddot{q}_{k-1}})$.

Define the tracking error of the joints as
\begin{align}
    \bm{\check{q}_k}=\bm{q_k}-\bm{r_k}
\end{align}
where $\bm{r_k}$ is the reference vector of the joints.

To obtain precise, fast and robust performance, we design a novel discrete-time terminal sliding surface as 
\begin{align}
    \bm{s_k}=\bm{a_1} \bm{\check{q}_k} +\bm{a_2} \textrm{sig}^{\bm{\beta_k}} \bm{\check{q}_k}+\bm{\dot{\check{q}}_k}
\end{align}
with $
    \textrm{sig}^{\bm{\beta_k}} \bm{\check{q}_k}=[\textrm{sig}^{\beta_{k,1}} \check{q}_{k,1}~\textrm{sig}^{\beta_{k,2}} \check{q}_{k,2}~\cdots~\textrm{sig}^{\beta_{k,n}} \check{q}_{k,n}]^{\bm{T}}
$, 
$
    \beta_{k,i}=\frac{|{\check{q}}_{k,i}|+0.5}{|{\check{q}}_{k,i}|+1},~i=1,2,\cdots,n
$,
where $\textrm{sig}^{\bullet}*=|*|^{\bullet}\textrm{sgn}(*)$, $\bm{a_1}=\textrm{diag}[a_{1,i}],~i=1,2,\cdots,n$, $\bm{a_2}=\textrm{diag}[a_{2,i}],~i=1,2,\cdots,n$, $a_{1,i}$ and $a_{2,i}$ are positive constants.

We design the reaching law of r-order variable-gain sliding mode controller as
\begin{align} \label{eq:reaching law}
    s_{k+1,i}=&-b_{0,i}(\ddot{q}_{k,i})T s_{k,i}^\eta-b_{1,i}(\ddot{q}_{k,i})Ts_{k-1,i}^\eta-\dots \nonumber \\
    &-b_{r,i}(\ddot{q}_{k,i})Ts_{k-r,i}^\eta
\end{align}
where $0<\eta<1$ is a positive constant, $b_{j,i}(\ddot{q}_{k,i}),~j=1,2,\cdots,r$ is a function of $\ddot{q}_{k,i}$. 


Based on the reaching law~(\ref{eq:reaching law}),the discrete-time r-order variable-gain terminal sliding mode controller with TDE is obtained as
\begin{align} \label{controller: overall}
    \bm{\tau_{k}}=&\frac{1}{T}\bm{\bar{M}}(\bm{\dot{r}_{k+1}}-a_1(\bm{\dot{q}_k}T+\bm{q_k}-\bm{r_{k+1}})-a_2\textrm{sig}^{\bm{\beta_k}}(\bm{\dot{q}_k}T
    \nonumber \\
    &+\bm{q_k}-\bm{r_{k+1}} ) -\bm{\dot{q}_{k}})+\bm{\bar{C}}(\bm{q_k},\bm{\dot{q}_k}) \bm{\dot{q}_k}+\bm{\bar{G}}(\bm{\dot{q}_k})
    \nonumber \\
    &+\bm{\bar{F_f}}(\bm{q}_k,\bm{\dot{q}}_k)+\bm{\hat{H}}(\bm{q_k},\bm{\dot{q}_k},\bm{\ddot{q}_k})
    \nonumber \\
    &+\bm{\bar{M}} ((1-\bm{b_0}(\bm{\ddot{q}_k})T)\bm{s_k}-\bm{b_1}(\bm{\ddot{q}_k})T\bm{s_{k-1}} \nonumber \\
    &-\bm{b_2}(\bm{\ddot{q}_k})T\bm{s_{k-2}} 
    -\dots -\bm{b_r}(\bm{\ddot{q}_k})T\bm{s_{k-r}}).
\end{align}


\section{Stability Analysis} \label{sec:stability analysis}

The following theorem states the sufficient condition to guarantee the control system stable:

\begin{mythm} \label{thm:theorem 1}
For the system~(\ref{model:discrete-time}) under the controller~(\ref{controller: overall}), if there are a series positive constants ${\alpha_1,~\alpha_2, ~\cdots, ~\alpha_r}$ satisfying 
$
     1> \alpha_1 >\alpha_2 >\cdots> \alpha_r >0
$
and 
\begin{align} \label{thm:conditions of b}
    b_{n,i}(\ddot{q}_{k,i}) \leq \frac{1}{T} \sqrt{\frac{\alpha_n-\alpha_{n+1}}{r+2}}
\end{align}
then $s_{k-m,i}$ will converge to the region
\begin{align}
    \Gamma_i= \left\{s_{k-m,i}\big | |s_{k-m,i}|<  \gamma_i \right\}
\end{align}
where 
\begin{align} \label{thm:region}
    \gamma_i&=\frac{(r+2)E_{k,i}^2+\sum_{j=0}^{r}(b_{j,i}(\ddot{q}_{k,i})T)^2}{\alpha_{m}-\alpha_{m+1}-(r+2)(b_{m,i}(\ddot{q}_{k,i})T)^2}
\\
    |s_{k-m,i}|&=\max \left\{ |s_{k-r,i}|,~|s_{k-r+1,i}|,~\cdots,~|s_{k,i}|\right\}
\end{align}
and $E_{k,i}$ is the elements of
$
     \bm{E_k}=\bm{\bar{M}}^{-1}\bm{\check{H}}(\bm{q_k},\bm{\dot{q}_k},\bm{\ddot{q}_k}).
$
\end{mythm}

\begin{proof}
Substituting (\ref{controller: overall}) into (\ref{model:discrete-time}), we obtain the dynamics of $s_{k+1,i}$ as $
    s_{k+1,i}=-b_{0,i}(\ddot{q}_{k,i})T s_{k,i}^\eta-b_{1,i}(\ddot{q}_{k,i})Ts_{k-1,i}^\eta-\dots 
    -b_{r,i}(\ddot{q}_{k,i})Ts_{k-r,i}^\eta+E_{k,i}.
$

Construct a Lyapunov candidate function $U_{k,i}$ as
\begin{align}
    U_{k,i}=s_{k,i}^2+\alpha_1 s_{k-1,i}^2+\alpha_2 s_{k-2,i}^2+\dots+\alpha_r s_{k-r,i}^2.
\end{align}
Then the difference between $U_{k+1,i}$ and $U_{k,i}$ is
\begin{align}
  \Delta U_{k}=& U_{k+1,i}-U_{k,i}
  \nonumber \\
  =&(\alpha_1-1)s_{k,i}^2+(\alpha_2-\alpha_1)s_{k-1,i}^2+\dots
    \nonumber \\
    &+(\alpha_r-\alpha_{r-1})s_{k-r+1,i}^2+s^2_{k+1,i}
    -\alpha_rs_{k-r,i}^2.
\end{align}

When $s_i>1$, we have $s_i^{2\eta}<s_i^2$.
When $s_i\leq 1$, we have $s_i^{2\eta}\leq 1$.
Thus, it holds that
\begin{align} \label{ineq:s}
    s_{i}^{2\eta}\leq s_i^{2}+1.
\end{align}

Based on the fact that $2ab\leq a^2+b^2,~a,b\in \mathbb{R}$, we have
\begin{align}
    s_{k+1,i}^2=&b_{0,i}^2(\ddot{q}_{k,i})T^2s_{k,i}^{2 \eta}+2b_{0,i}(\ddot{q}_{k,i})b_{1,i}(\ddot{q}_{k,i})T^2s_{k,i}^{\eta}s_{k-1,i}^{\eta}
    \nonumber \\
    &+(b_{1,i}(\ddot{q}_{k,i})T)^2s_{k-1,i}^{2\eta}  \nonumber \\
    &+2b_{1,i}(\ddot{q}_{k,i})b_{2,i}(\ddot{q}_{k,i}) T^2s_{k-1,i}^{\eta}s_{k-2,i}^{\eta}
  \nonumber \\
    &+(b_{2,i}T)^2s_{k-2,i}^{2\eta}+\dots 
    +(b_{r,i}(\ddot{q}_{k,i})T)^2s_{k-r,i}^{2\eta}
      \nonumber \\
    &+(b_{r,i}(\ddot{q}_{k,i})T)s_{k-r,i}^{\eta}E_{i}+E_{i}^2
    \nonumber \\
    \leq& (r+2)[(b_{0,i}(\ddot{q}_{k,i})T)^2s_{k,i}^{2\eta}+(b_{1,i}(\ddot{q}_{k,i})T)^2s_{k-1,i}^{2\eta}
    \nonumber 
    \\&+\dots+(b_{r,i}(\ddot{q}_{k,i})T)^2s_{k-r,i}^{2\eta}+E_{i}^2].
\end{align}    

According to~(\ref{ineq:s}) further, we have
\begin{align}
    s_{k+1,i}^2
    <& (r+2)\left((b_{0,i}(\ddot{q}_{k,i})T)^2s_{k,i}^{2}+(b_{1,i}(\ddot{q}_{k,i})T)^2s_{k-1,i}^{2} \right. \nonumber 
    \\&\left. +\dots+(b_{r,i}(\ddot{q}_{k,i})T)^2s_{k-r,i}^{2}\right)
    \nonumber \\
    &+(r+2)\left(E_{i}^2+\sum_{j=0}^{r}(b_{j,i}(\ddot{q}_{k,i})T)^2\right).
\end{align}



If we define $\alpha_0=1$ and $\alpha_{r+1}=0$, then there is
\begin{align}
    \Delta U_{k,i}=&\left(\alpha_1-\alpha_0+(r+2)(b_{0,i}(\ddot{q}_{k,i})T)^2\right)s_{k,i}^2 \nonumber \\
    &+\left( \alpha_2-\alpha_1 +(r+2)(b_{1,i}(\ddot{q}_{k,i})T)^2\right)s_{k-1,i}^2+\dots 
    \nonumber \\
    &+\left( \alpha_r-\alpha_{r-1}+(r+2)(b_{r-1,i}(\ddot{q}_{k,i})T)^2\right)s_{k-r+1,i}^2 
    \nonumber \\
    &+(\alpha_{r+1}-\alpha_r+(r+2)(b_{r,i}(\ddot{q}_{k,i})T)^2)s_{k-r,i}^2
    \nonumber \\
    &+(r+2)E^2.
\end{align}

From~(\ref{thm:conditions of b}), we obtain that 
\begin{align}
    \alpha_{n+1}-\alpha_{n}+(r+2)(b_{n,i}(\ddot{q}_{k,i})T)^2s_{k,i}^2 \leq 0.
\end{align}

Assume that $s_{k-m, i}=\max \{s_{k,i},s_{k-1,i},\dots, s_{k-r,i}\}$, then we have
\begin{align}
    \Delta U_{k,i}
    &\leq \left(\alpha_{m+1,i}-\alpha_{m,i}
    +(r+2)(b_{m,i}(\ddot{q}_{k,i})T)^2\right)s_{k-m,i}^2
    \nonumber \\
    &
    +(r+2)E^2+\sum_{j=0}^{r}(b_{j,i}(\ddot{q}_{k,i})T)^2.
\end{align}

When $|s_{k-m}|$ is not in the region $\Gamma_i$, i.e., there is 
\begin{align}
    s_{k-m,i}^2
    &>\frac{(r+2)E^2+\sum_{j=0}^{r}(b_{j,i}(\ddot{q}_{k,i})T)^2}{\alpha_{m}-\alpha_{m+1}-(r+2)(b_{m,i}(\ddot{q}_{k,i})T)^2}
\end{align}
it is obtained that
\begin{align}
    \Delta U_{k,i}<0.
\end{align}
This completes the proof of Theorem~\ref{thm:theorem 1}.
 \end{proof}

\begin{myrem}
From~(\ref{thm:conditions of b}), we can infer that when the sampling interval $T$ is smaller, $b_{n,i}(\ddot{q}_{k,i})$ has a larger range of values. In other words, the stability of the system determines a maximum value of the sampling interval. 
\end{myrem}

\begin{myrem}
Equation~(\ref{thm:region}) gives some clues on the source of errors. Firstly, a larger $T$ results in a larger region~$\Gamma_i$. Secondly, a larger $E_{k,i}$ also causes a larger convergence region of $s_{k-m,i}$, and worse tracking precision further. As $E_{k,i}$ is related to the estimation error of TDE, we can get that a precise TDE is the precondition of the precise tracking. Traditional methods such as~\cite{lee2017adaptive} and~\cite{wang2016practical}, only rely on TDE to compensate for the nonlinearity. They have a larger estimation error than our method under the situations that the nonlinearity is large. Thus our method has better precision theoretically. 
\end{myrem}

\section{SIMULATIONS AND EXPERIMENTS} \label{sec:implementation and simulation}

\begin{figure}[http]
  \centering
   \includegraphics[width=220pt]{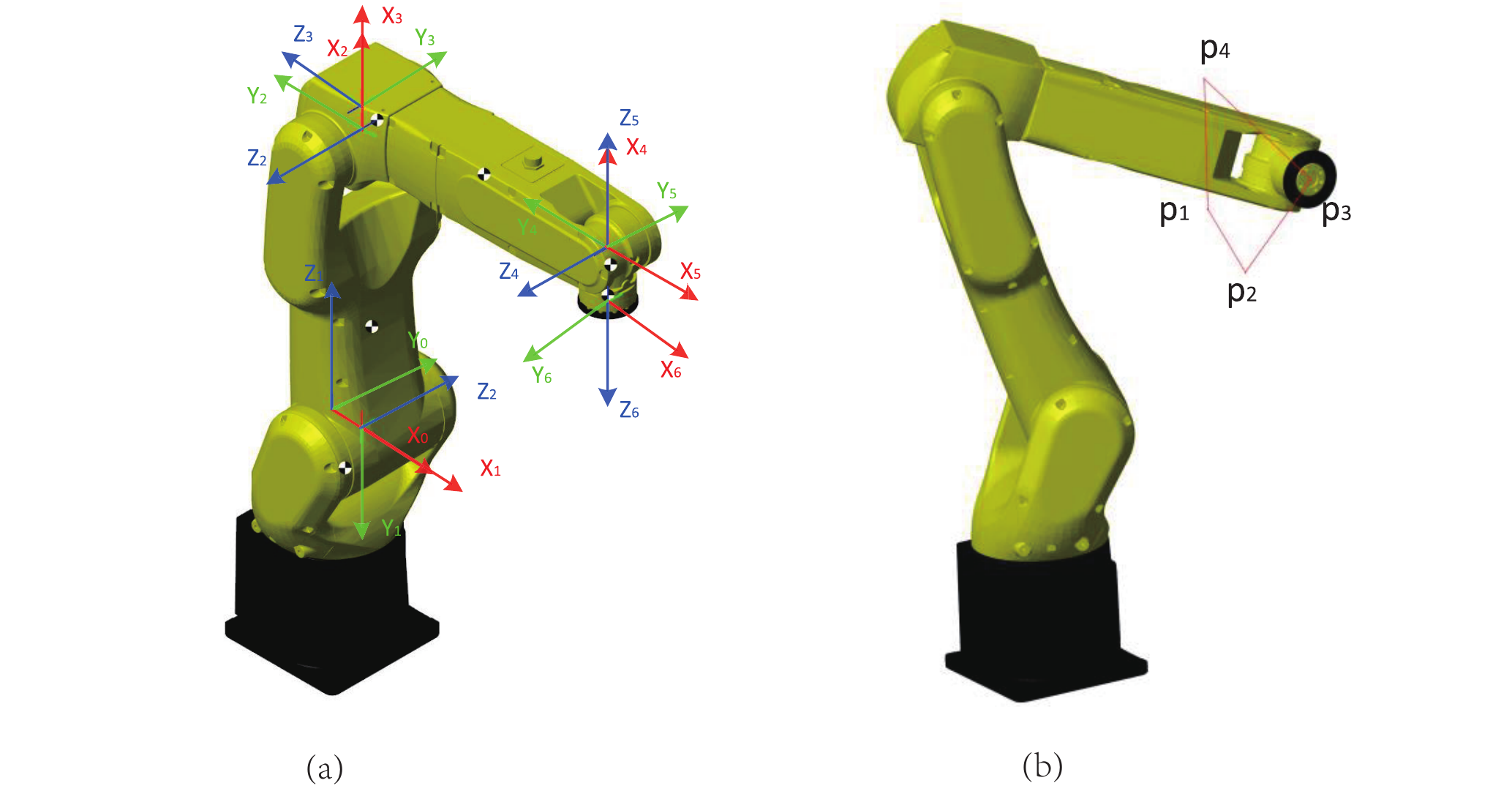}
   \caption{The simulation setup of the FANUC LR Mate 200iD industrial robot, including (a) the frames and (b) the reference trajectory.}
   \label{fig:robot}
   \end{figure}

      \begin{figure}[http]
  \centering
   \includegraphics[width=220pt]{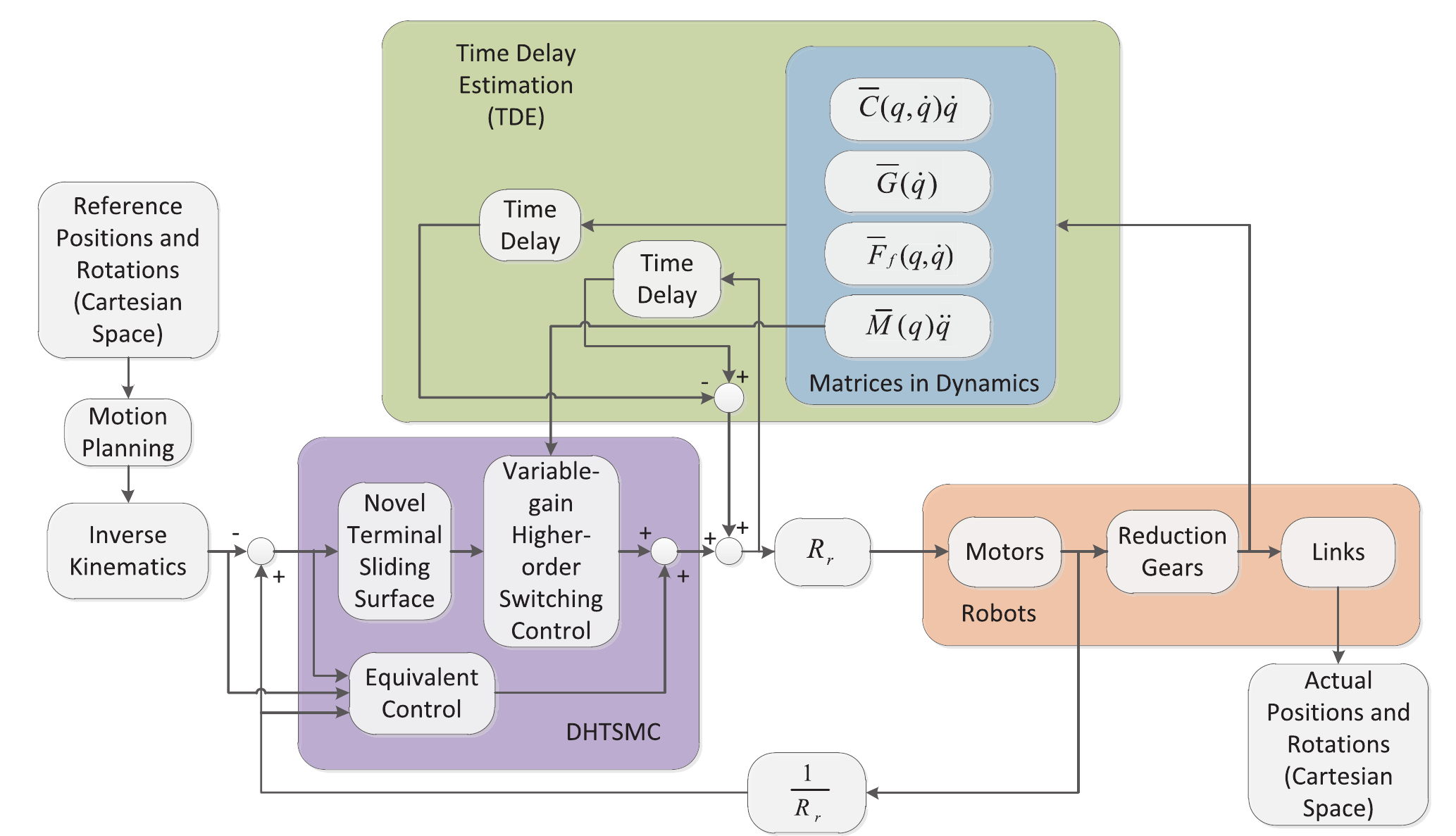}
   \caption{Structure of the closed-loop system where DHTSMC is applied to robots~($R_r$ stands for the reduction ratios of gears). }
   \label{fig:structure}
   \end{figure}
   


\begin{table}[tp]  
    \centering
    \caption{D-H Parameters of 6-DOF FANUC LR Mate 200iD Robot}
    \label{tab:DH parameters of LR mate}
  \begin{tabular}{ccccc}
    \cline{1-5}
    $i$ & $\alpha_{i-1}$~(rad) & $a_{i-1}$~(mm) & $d_i$~(mm) & $\theta_i$~(rad) \\
    \cline{1-5}
    1 & -$\frac{\pi}{2}$ & 50 & 0 & $q_1$\\
    
    2 & $\pi$ & 440 & 0 & $q_2-\frac{\pi}{2}$\\
    
    3 & -$\frac{\pi}{2}$  & 35 & 0 & $q_3$\\
    
    4 & $\frac{\pi}{2}$ & 0 & -420 & $q_4$\\
    
    5 & -$\frac{\pi}{2}$ & 0 & 0 & $q_5$\\
    
    6 & $\pi$ & 0 & -80 & $q_6$\\
    
\cline{1-5}
      \end{tabular}
    \end{table}
    \linespread{1}

\begin{table*}[tp]  
    \centering
    \caption{Physical Parameters of 6-DOF FANUC LR Mate 200iD Robot}
    \label{tab:physical parameters of LR mate}
  \begin{tabular}{cccccccc}
    \cline{1-8}
    \multirow{2}*{NO. } & \multirow{2}*{Gear Ratio} &  Link Mass  & Motor Inertia  & Link Coulomb   & Link Viscous   & Motor Coulomb  & Motor Viscous   \\
    
    &&(kg)&(kg m$^2$)& Friction (N$\cdot$m)& Coefficient  (N$\cdot$s/m)& Friction (N$\cdot$m)& Coefficient (N$\cdot$s/m) \\
    \cline{1-8}
    1 & 114.6 & 2.4 & 8.9 $\times$10$^{-5}$ & 0.045 &3.1 & 0.052&0.23$\times$10$^{-3}$\\
    
    2 & 121 & 7.8 & 6.0 $\times$10$^{-5}$ & 0.095 &4.1&0.052&0.28$\times$10$^{-3}$\\
    
    3 & 102.1  & 3.0 & 5.2 $\times$10$^{-5}$ & 0.11 &1.2&0.041&0.11$\times$10$^{-3}$\\
    
    4 & 73.0 & 4.1 & 7.2 $\times$10$^{-5}$ & 0.094 &0.81&0.044&0.15$\times$10$^{-3}$\\
    
    5 & 83.3 & 1.7 & 1.4 $\times$10$^{-5}$ & 0.10 &0.37&0.012&0.053$\times$10$^{-3}$\\
    
    6 & 41.4 & 0.17 & 1.7 $\times$10$^{-5}$ & 0.15 &0.18&0.028&0.11$\times$10$^{-3}$\\
    
\cline{1-8}
      \end{tabular}
    \end{table*}
    \linespread{1}

\subsection{Simulation}

The simulation environment is developed based on MATLAB R2019b and Simulink. With the help of the tool of Mechanics Explorer, a complete model of FANUC LR Mate 200iD industrial robot~(shown in Fig.~\ref{fig:robot}) is established. Coordinates of the manipulator are set as shown in Fig.~\ref{fig:robot} (a), and corresponding D-H parameters are displayed in Table~\ref{tab:DH parameters of LR mate}. Other related physical parameters of the simulation model are given in Table~\ref{tab:physical parameters of LR mate}.

The control structure of the proposed method is shown in Fig.~\ref{fig:structure}. First of all, the reference positions and rotations in Cartesian space are generated. The initial position is $p_1(0.47,0,0.395)$, and the initial rotation angle is $\psi_1(0,0,\pi)$ in the order of $Z_1Y_2X_3$. Then it moves to $p_2(0.2,0.2,0.2)$ with rotation angle $\psi_2(\pi/4,\pi/4,\pi/2)$. At point $p_3(0.2,0.3,0.3)$ with rotation angle $\psi_3(\pi/2,\pi/2,\pi/2)$, after that, the manipulator returns to the initial position and rotation via $p_4(0.3,0.1,0.5)$ with rotation $\psi_4(\pi/4,\pi/4,\pi/2)$, as shown in Fig.~\ref{fig:robot} (b).

Then the motion planning rules are conducted to smooth the trajectories and realize some additional functions on robots. The simple S-curve motion planning rule is considered in the simulation to make the acceleration and deceleration of the robots are performed gently and steadily. 
After that, we input the planning results to the inverse kinematics, which is derived based on the frames and Denavit–Hartenberg~(D-H) parameters of the robot manipulator, to obtain the reference $\bm{q}$. Then the proposed controller is applied in the joint space. 

\begin{figure}[http]
  \centering
   \includegraphics[width=250pt]{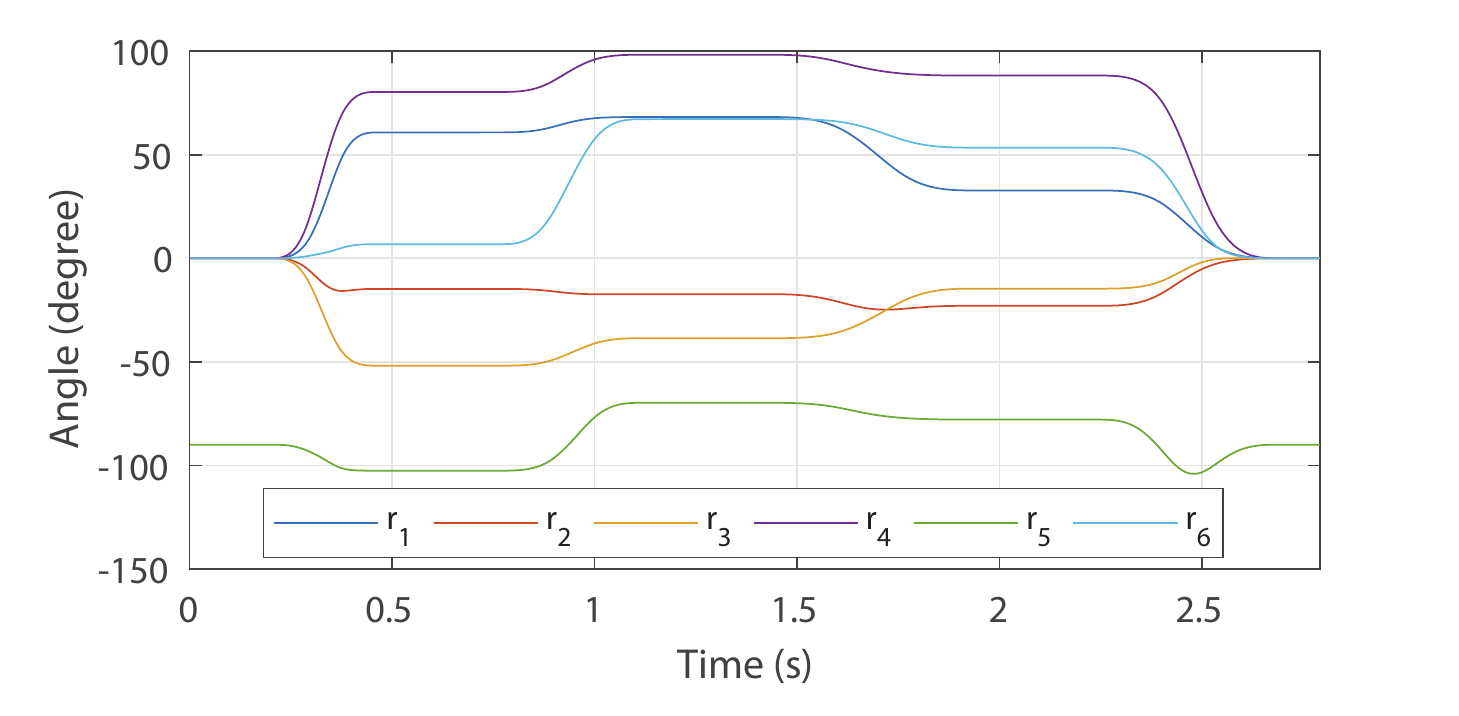}
   \caption{The reference angle of each joint in the simulation.}
   \label{fig: reference_simulation}
   \end{figure}

\begin{figure}[http]
  \centering
   \includegraphics[width=250pt]{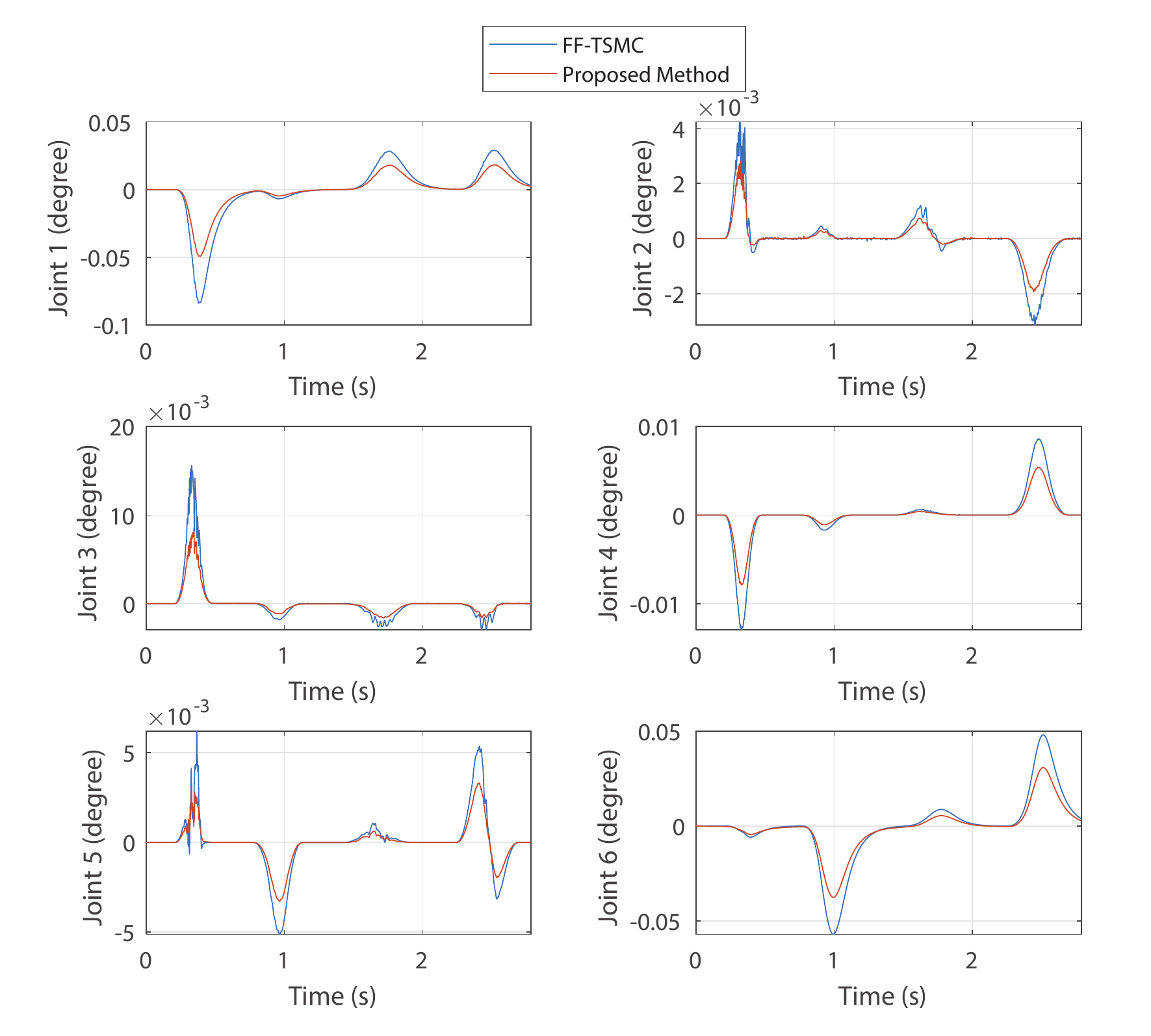}
   \caption{The tracking error in the joint space in the simulation.}
   \label{fig: sim_joint_error}
   \end{figure}

We set the maximum acceleration as $50~m/s^2$ from $p_1$ to $p_2$ and as $10~m/s^2$ from $p_2$ to $p_1$. Before every motion, the manipulator will wait for $0.1~s$, and the sampling time is set as the same as the controller, $1~ms$. After that, we calculate the inverse kinematics based on the frames shown in Fig.~\ref{fig:robot} (a) and D-H parameters shown in Table.~\ref{tab:DH parameters of LR mate}. After the planning and calculation of inverse kinematics, the trajectory of each joint is shown in Fig.~\ref{fig: reference_simulation}. We introduce the band-limited white noise as an additional disturbance, whose sampling time is set as $0.1$~s and other parameters remain default. For the controller, parameters are tuned as
 $a_1=diag[10, 100, 100, 15, 100, 10]$, $a_2=0.01 \bm{I}$, $b_0=4.5\times 10^5+0.005 \ddot{q}_i$, $b_1=2.25 \times 10^5$, $r=1$ and $\bm{I}$ is the unit matrix.

Apart from the proposed DHTSMC method, we also implement the traditional feedforward control scheme with TSMC (noted as FF-TSMC) with TDE proposed in~\cite{wang2020model} as a comparison. The results of the simulations are illustrated in Fig.~\ref{fig: sim_joint_error}. For each joint, there are four peaks or valleys in the trajectory of tracking error, corresponding to the processes moving from one point to the other. When the joint is waiting in a position, tracking errors of both method can converge to a relatively small value. This phenomenon means that when the joint is accelerating or decelerating, the joint tends to have a more massive tracking error. However, with the proposed method, the tracking error is significantly reduced when the joints change their positions, which is about two-thirds of the traditional FF-TSMC. This shows that compared with the feedforward structure, the proposed method can make better use of the dynamics, to get better performance.

  \begin{figure}[http]
  \centering
   \includegraphics[width=250pt]{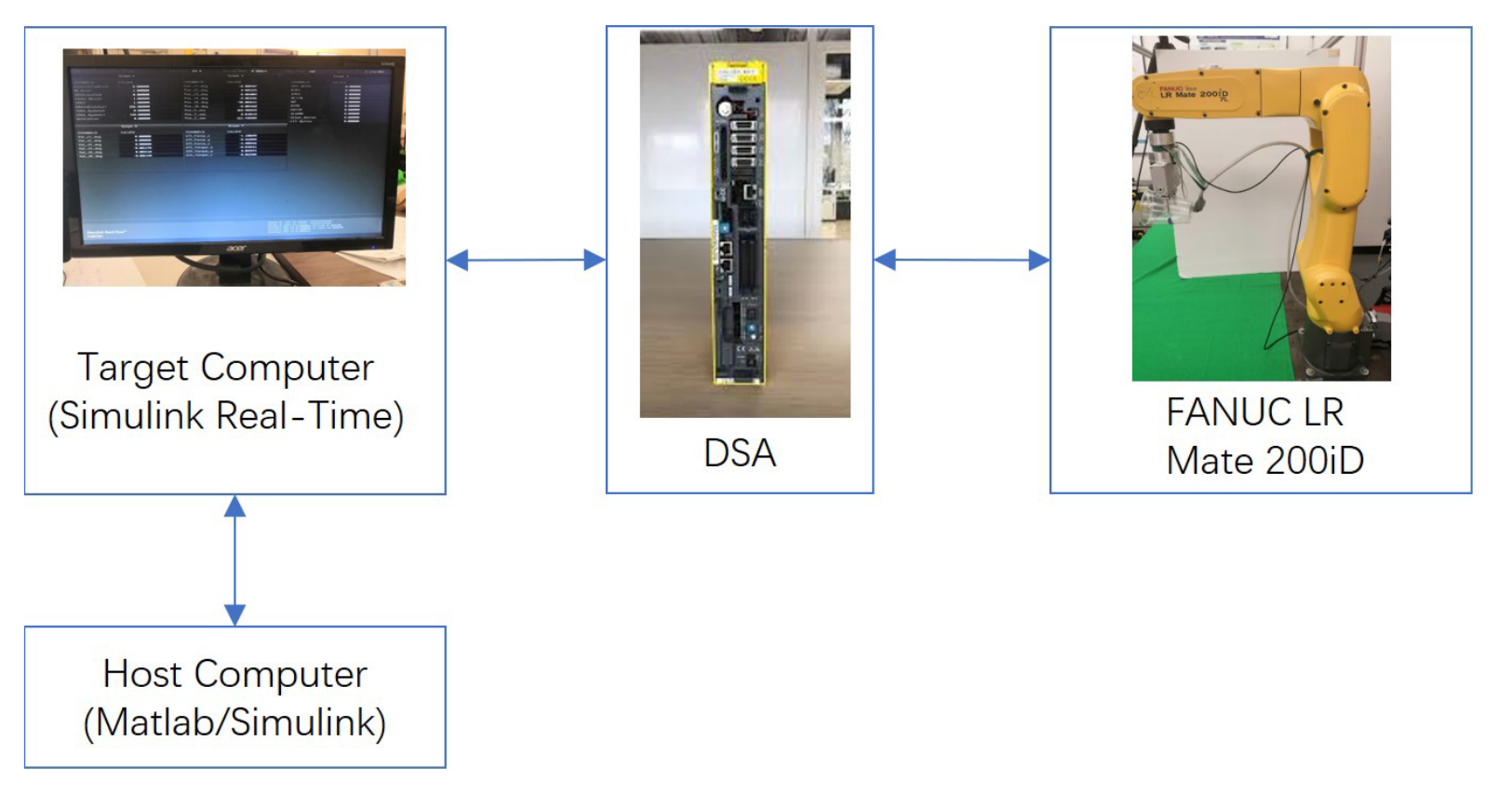}
   \caption{The experimental setup.}
   \label{fig: experimental setup}
   \end{figure}
  

\begin{figure}[http]
  \centering
  \includegraphics[width=250pt]{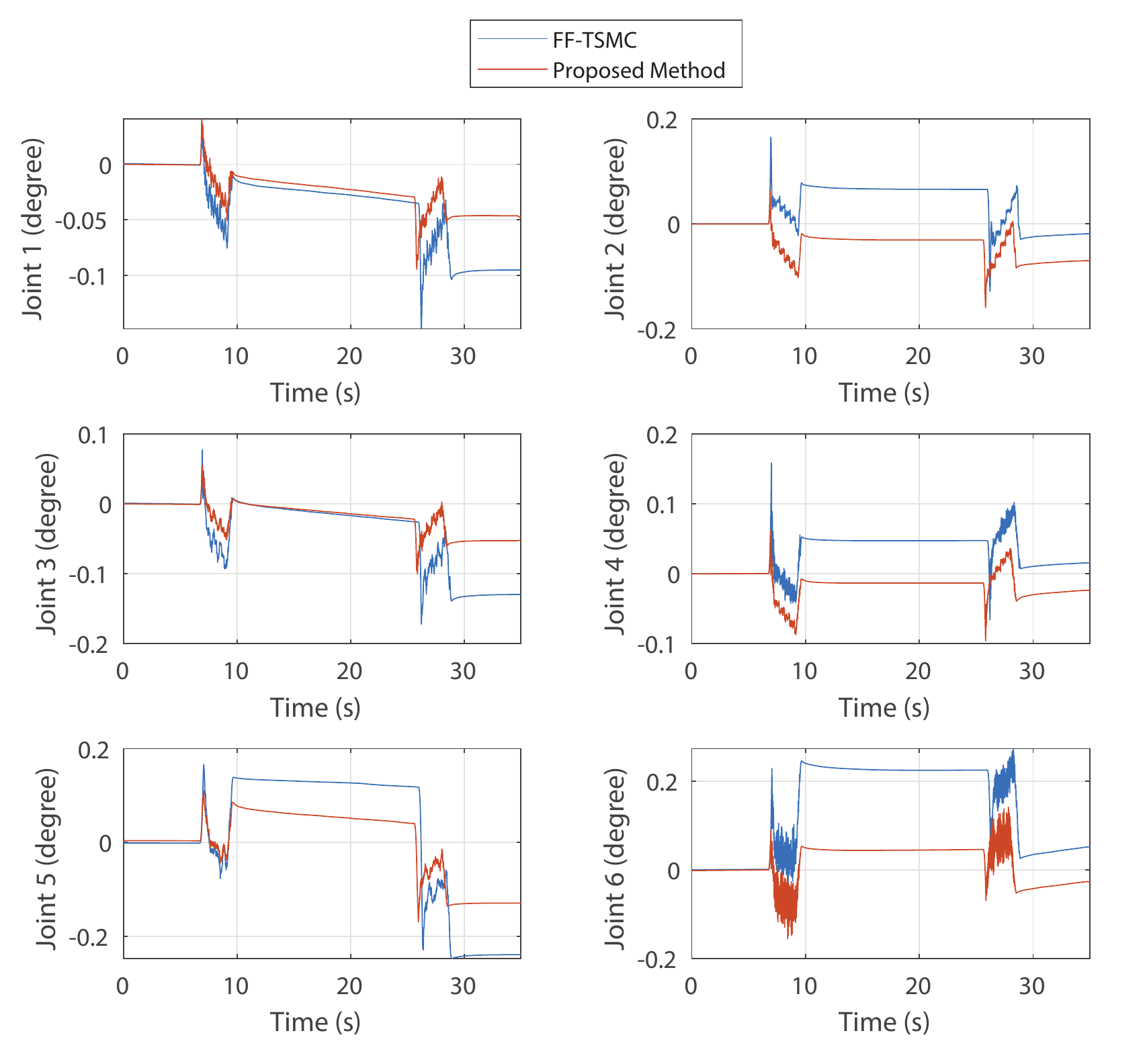}
  \caption{The tracking error in the joint space in the experiments.}
  \label{fig: joint error experiments}
  \end{figure}

  \begin{figure}[http]
  \centering
  \includegraphics[width=250pt]{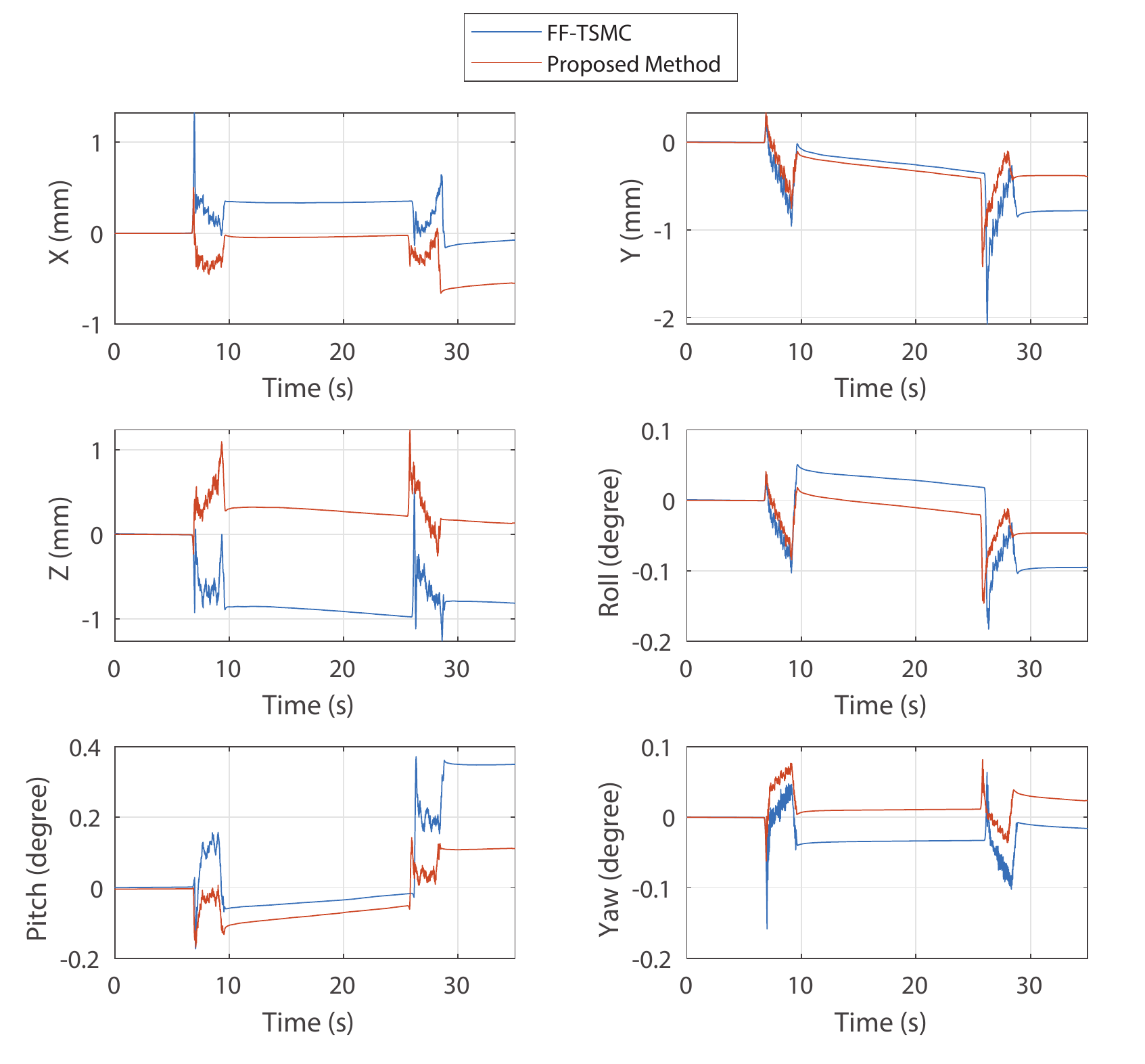}
  \caption{The tracking error in the Cartesian space in the experiments.}
  \label{fig: Cartesian experimental}
  \end{figure}

\subsection{Experiments}

Practical experiments are also conducted to verify the effectiveness of the proposed method further. As shown in Fig.~\ref{fig: experimental setup}, the experimental setup consists of the host computer, the target computer, the digital servo adapter (DSA), and the FANUC LR Mate 200iD robots. The control algorithm is programmed on the host computer and applied to the DSA and robot via Simulink Real-Time running on the target computer. The parameters of the robot are the same as the model in the previous simulation. The sampling interval of the robot is set as 0.25 ms, and that of the controller is set as 1 ms. 

We also compare the FF-TSMC and the proposed control method in the experiments. The reference trajectory of the robots is directly given in joint space. All the joints are set to rotate from $0^\circ$ to $20^\circ$ and then return to $0^\circ$. Parameters of the controller are set as $a_1=diag[1, 20, 13, 2, 15, 3]$, $a_2=0.015 \bm{I}$, $b_0=1\times 10^5+0.002 \ddot{q}_i$, $b_1=0.25 \times 10^5$ and $r=1$.

Fig.~\ref{fig: joint error experiments} shows the tracking error of each joint under the two controllers. We can see that the tracking error of the proposed method is overall smaller in each joint compared with the previous FF-TSMC. We have also noticed that for both methods, there is an offset after the manipulator reaches the target position. We think this is caused by the unmodelled dynamics of the robot. It is easy to find that for most joints, the proposed method has a smaller offset than the FF-TSMC. 

Further, the performance of the robots in the Cartesian space is shown in Fig.~\ref{fig: Cartesian experimental}. The position and pose of the end-effector are
calculated by using the joint position and the forward kinematics. In particular, the pose is described by Euler angles, which rotates in the order of $Z_1Y_2X_3$. From Fig.~\ref{fig: Cartesian experimental}, we can see that the proposed method has a smaller error in both positions and poses. The position error of the proposed method is around 1 mm, and the Euler angular error is below 0.2$^\circ$, which are half of the previous method.

\section{CONCLUSIONS}

A control method of DHTSMC with TDE was proposed and implemented on the manipulators in this paper. The controller was derived with the nonlinear dynamics of robots taken into the feedback loop directly, so that the nonlinearities, the uncertainties, and disturbances were adequately handled. Specifically, we improved the structure of the sliding surface to achieve a faster response, and we gave a universal control method of different order sliding mode controllers. By using the variable gains, the conflict between the fast response and small chattering amplitude was further settled. We also analyzed the stability of the closed-loop system. Results of simulations and experiments showed that the proposed controller was much more effective than conventional DTSMC methods in both joint space and Cartesian space. 






\section*{ACKNOWLEDGMENT}
The authors would like to thank Dr.~Yu Zhao for his contribution to the MATLAB simulator of the robot manipulator used in this work.



\bibliographystyle{IEEEtran}
\balance
\bibliography{IEEEabrv}

\vfill

\end{document}